\documentclass{article}

\usepackage{arxiv}

\usepackage[utf8]{inputenc} % allow utf-8 input
\usepackage[T1]{fontenc}    % use 8-bit T1 fonts
\usepackage{hyperref}       % hyperlinks
\usepackage{url}            % simple URL typesetting
\usepackage{booktabs}       % professional-quality tables
\usepackage{amsfonts}       % blackboard math symbols
\usepackage{nicefrac}       % compact symbols for 1/2, etc.
\usepackage{microtype}      % microtypography
\usepackage{lipsum}
\usepackage{times}  %Required
\usepackage{helvet}  %Required
\usepackage{courier}  %Required
\usepackage{url}  %Required
\usepackage{graphicx}  %Required

\usepackage{amsmath}
\usepackage{amsthm}
\usepackage{amsfonts, mathrsfs}
\usepackage{amssymb}
\usepackage{pgflibrarysnakes}
\usepackage{float}
\usepackage{graphicx}
\usepackage{multirow}
\usepackage{epsfig}
\usepackage{epstopdf}
\usepackage{verbatim}
\usepackage{enumerate}
\usepackage{subcaption}
\usepackage{color}
\usepackage{natbib}
\usepackage{tensor}
\usepackage{url}
\usepackage{tikz}
\usepackage{setspace}
\usepackage{color}
\usepackage[linesnumbered,ruled]{algorithm2e}

\allowdisplaybreaks
\numberwithin{equation}{section}
\newtheorem{theorem}{Theorem}

\newtheorem{proposition}[theorem]{Proposition}
\newtheorem{lemma}[theorem]{Lemma}
\newtheorem{example}[theorem]{Example}

\newtheorem{definition}[theorem]{Definition}

\DeclareMathOperator*{\argmax}{\rm argmax}
\DeclareMathOperator*{\argmin}{\rm argmin}

\title{Dynamic Learning of Sequential Choice Bandit Problem under Marketing Fatigue}

\author{
  Junyu Cao \\
  University of California, Berkeley\\
  Berkeley, CA 94720 \\
  \texttt{jycao@berkeley.edu} \\
   \And
  Wei Sun\\
  IBM Research\\
  Yorktown Height, New York 10591 \\
  \texttt{sunw@us.ibm.com} \\
}
\date{}

\begin{document}
\maketitle

\begin{abstract}
Motivated by the observation that overexposure to unwanted marketing activities leads to customer dissatisfaction, we consider a setting where a platform offers a sequence of messages to its users and is penalized  when users abandon the platform due to marketing fatigue. We propose a novel sequential choice model to capture multiple interactions taking place between the platform and its user: Upon receiving a message, a user decides on one of the three actions: accept the message, skip and receive the next message, or abandon the platform.  Based on user feedback, the platform dynamically learns users' abandonment distribution and their valuations of messages to determine the length of the sequence and the order of the messages, while maximizing the cumulative payoff over a horizon of length $T$. We refer to this online learning task as the \emph{sequential choice bandit} problem. For the offline combinatorial optimization problem, we show that an efficient polynomial-time algorithm exists. For the online problem, we propose an algorithm that balances exploration and exploitation, and characterize its regret bound. Lastly, we demonstrate how to extend the model with user contexts to incorporate personalization.
\end{abstract}

% keywords can be removed

\section{Introduction}

Service providers and retailers routinely rely on emails and app notifications to interact with their users. When it is done well, these messages act as digital reminders that increase customer engagement, raise brand awareness and conversion. However, frequent messaging can easily backfire. Marketing fatigue, which refers to an overexposure to unwanted marketing messages, could aggravate users and prompt them to forgo receipt of future messages by unsubscribing or deleting the app.  

Motivated by this dilemma, we consider a setting where a platform needs to learn a  policy which consists of a sequence of messages for its users. It has to decide  \emph{the order of the messages} as well as \emph{the length of the sequence} from a pool of available messages. The messages are presented to a user sequentially. Upon reviewing a message $i$, a user takes one of the three actions: 1) accept the message and exit. In this case, the platform earns a reward $r_i$. If the user does not select the current  message, she can either 2) receive the next message unless the sequence runs out, or 3) abandon the platform. When a user abandons,  the platform incurs a penalty cost $c$ from losing that user. 

Based on users' feedback, the platform learns two pieces of information in order to determine the optimal sequence, namely, users' valuations of individual messages and users' abandonment distribution. The objective of the platform is to maximize its expected payoff which is the revenue after subtracting the penalty cost due to abandonment. We refer to the online learning task which the platform faces as the \emph{sequential choice bandit} (\emph{SC-Bandit}) problem. 

To draw a connection between this problem and the earlier motivating example, messages can represent digital marketing content such as an email or app notification regarding a product or service that a marketer wishes to promote. He\footnote{We refer to a marketer as \emph{he}, and a user as \emph{she}.} earns revenue whenever a user interacts with the content (e.g., click or purchase). The interaction is an indication that the content is of interest to that user. When the user ignores the content, there is a possibility that she will unsubscribe or delete the app. We can think of the abandonment cost $c$ as the cost of user acquisition as the marketer replenishes his customer base. Based on a survey\footnote{https://hbr.org/2014/10/the-value-of-keeping-the-right-customers},  the cost of customer acquisition is estimated to be 5 to 25 times higher than keeping an existing customer. Therefore, fatigue control is a critical component of digital marketing content dissemination.

There are several challenges associated with analyzing the SC-Bandit problem. Firstly, even in the offline setting where users' valuations and abandonment distribution are known, the sequence optimization problem is  combinatorial in nature without an obvious efficient algorithm. Secondly, the sequential behavior of users complicates the learning task: while one can observe the response to the first offered message, the feedback to subsequent messages is not guaranteed due to abandonment. Thirdly, one needs to simultaneously learn valuations and abandonment distribution from users' feedback which depends on these two pieces of information jointly. The contribution of our work is fourfold: 
\begin{enumerate}
    \item We propose a novel sequential choice model which captures multiple interactions including abandonment between users and a platform.
    \item We prove that the offline combinatorial optimization problem allows an efficient polynomial-time algorithm.
    \item For the online problem where valuations and abandonment distribution are unknown to the platform, we propose a learning algorithm and show that the regret  is bounded above by $O(N\sqrt{T\log T})$ where $N$ is the number of available messages and $T$ is the time duration.
    \item We incorporate personalization by solving a contextual SC-Bandit problem where valuations and abandonment distribution can vary with user features.
\end{enumerate}

\section{Literature review}

\paragraph{Multi-armed bandit problem} Our work is closely related to the multi-armed bandit (MAB) problem, which has been well studied in literature (e.g., \citealp{robbins1985some}; \citealp{sutton1998reinforcement}). Several popular extensions include MAB with linear payoffs (\citealp{auer2002using,agrawal2013thompson}), ranked bandits (\citealp{radlinski2008learning,slivkins2013ranked}), and combinatorial MAB problem (\citealp{chen2013combinatorial}). Our problem can be viewed as a combinatorial bandit problem where a platform chooses a set of messages to be displayed in a certain order. 
%Our problem falls under the combinatorial setting since the platform's decision includes both the choice of the messages and their order. 
A naive approach is to treat each possible combination as an arm. However, the number of arms increases exponentially with the number of messages under this approach. Other combinatorial bandit work assuming linear reward (\citealp{auer2002using, rusmevichientong2010linearly}) or independent rewards (\citealp{chen2013combinatorial}) cannot be directly applied to our model.
%\todo{delete} (A number of studies on combinatorial MAB consider playing $k$ arms among $m$ arms simultaneously (\citealp{caro2007dynamic,agrawal2017thompson}). Compared to these work, our problem considers a novel sequential setting where arms  have to be played in a given order and feedback can be incomplete as users can abandon before viewing the entire sequence. ) \todo{cascading?}
Our setup also shares some similarities with cascading bandits (\citealp{kveton2015cascading}). The task there is to select $m$ messages with the highest click probabilities, where $m$ is exogenous and the rewards are the same for all messages. In contrast, our task is to determine both $m$ (the length of the sequence) and the order of the messages which have different revenues.

Some recent work such as  \cite{schmit2018learning} has studied users' abandonment. %\cite{kanoria2018managing} considers a platform which can take two actions to mitigate abandonment. 
In their setting, a user has a threshold drawn from an unknown distribution and she abandons if the platform's action $x$ exceeds that threshold. The platform needs to learn the distribution while optimizing $x$ to maximize its discounted reward. One of the key differentiators and novelty of our work is how we model abandonment in the presence of sequential behavior. The decision to abandon is an interplay of user's valuations which determine whether a user will select the message, and the abandonment distribution. The platform needs to learn both quantities and solve an integer programming problem to obtain the optimal policy.

\paragraph{Dynamic learning of assortment optimization problems} Assortment optimization refers to the problem of selecting a set of products to offer to a group of customers so as to maximize the revenue that is realized when customers make purchases according to their preferences. It is a central topic in the operations management research literature. We refer the reader to \cite{kok2008assortment} for a comprehensive review.
\cite{talluri2004revenue} formulate the assortment planning problem by using a discrete choice model which is a multinomial logit model~(\citealp{train2009discrete,luce2012individual}) to describe user behavior. 

More recent literature such as (\citealp{caro2007dynamic,rusmevichientong2010dynamic, agrawal2017thompson,cheung2017thompson}) focus on the dynamic assortment problem where the customer preferences are unknown a priori and need to be learnt. Our work can be viewed as a dynamic assortment problem to determine a set of messages and a specific display order.  Existing dynamic assortment problems model a single interaction between the platform and a user, who can either choose an item from the assortment or leave without a purchase. In contrast, our model captures multiple interactions between the two - the sequential nature of the decision-making process is a key novelty of our work. The order of arrival of messages plays a crucial role in the analysis as message rewards vary and users could abandon the platform when unsatisfying messages are received.

\section{Model}
%In e-commerce, customers browse products much more frequently than previous years with the rapid increasing popularity of smart phones, which also provides fruitful information for the retailer to predict customers' purchase behavior. However, when browsing the website, they are also very easily to be distracted by other information, so it will be beneficial for the retailer to send a notification to offer another assortment to remind them. On the other hand, if the retailer sends too many notifications, customers may get tired of the overwhelming information. Here we propose a  sequential multinomial logit model with customer fatigue, which we call SMNL model. 

In this section, we formally introduce our setting. Assume there are $N$ different messages for the platform to choose from. Denote $X$ as the set of these  $N$ messages. Each message $i$ generates revenue $r_i$ when it is selected by a user.  Customers arrive at time $t=1,\cdots,T$. For a customer arriving at time $t$, the platform determines a sequence of messages $\bold{S}^t=S_1^t\oplus S_2^t\oplus\cdots\oplus S_m^t$, where $S_i^t$ consists of a single message for any $1\leq i\leq m$ and ``$\oplus$" denotes the operator of union which also preserves the order. The platform's decision includes both the order of the messages as well as the total length of the sequence $m$.

Messages are displayed sequentially to a user according to the pre-specified order. Thus, messages {at the front } of the sequence will be displayed first and  are considered having higher  priorities. If a user selects a message, she exits the platform and no further messages will be shown to her. The platform earns $r_i$. On the other hand, when a message is not selected, 
we consider its content unsatisfying, since they are not of sufficient interest to the user. When that happens, the user can either choose to abandon the platform, or see the next message until the sequence runs out. Abandonment will cause a penalty cost $c$ to the platform. 

\paragraph{Abandonment distribution under marketing fatigue}  We assume the probability that a user abandons the platform upon receiving each unsatisfying message is $p$. Each user arriving at time $t$ can be characterized by a random variable $W^t$, drawn from a distribution $F_W$. $W^t$ is a proxy for user patience, which measures the maximum number of unsatisfying messages that a user can tolerate before abandoning the platform. Under this setup, it implies that $F_W$ is a geometric distribution with parameter $p$. Let $q=1-p$.

 The  probability of upon receiving the $k^{th}$ unsatisfying message is $P(W=k)=q^{k-1}(1-q)$. The probability that a user has not abandoned after $k$ unsatisfying messages is $P(W > k)=q^k$, which is also the probability that a user's patience is larger than $k$.

%Denote $q_k$ as the abandonment probability upon receiving the $k^{th}$ unsatisfying message, i.e., $q_k=P(W=k)=q^{k-1}(1-q)$.  $\bar{F}_W(k)=P(W > k)=1-P(W\leq k)=q^{k}$ is the probability that a user's patience is larger than $k$, i.e., the user has not abandoned after $k$ unsatisfying messages.

%We use $\bar{F}_W(l)$ to denote the tail distribution which is the probability that a user has not abandoned after $l$ messages, i.e., $\bar{F}_W(l)=1- F_W(l)=P(W > l)=q^{l}$.

\paragraph{Sequential choice model} For every message $i$, its probability of being selected is $u_i$, where $0\leq u_i<1$. This quantity can be directly derived  from users' valuation of message $i$ which reflects users' preferences. For the rest of the paper, we will refer to $u_i$ as valuation to avoid confusion with $p_i(\bold{S})$ which we will define next. When message $i$ is part of a sequence $\bold{S}=S_1\oplus S_2\oplus\cdots\oplus S_m$, the probability of being selected which is denoted as $p_i(\bold{S})$, depends on its position in the sequence as well as the content of other messages shown earlier. Formally,

\begin{equation}
 p_{i}(\bold{S})=
\left\{
\begin{aligned}
&u_i,& \text{ if } i\in S_1\\
&P(W\geq l)\prod_{k=1}^{l-1}(1-u_{I(k)})u_i,& \text{ if } i\in S_l, l\geq 2\\
&0, &\text{ if }i\notin \bold{S},
\end{aligned}
\right.\notag
\end{equation} 
where $I(\cdot)$ denote the index function that $I(k)=i$ if and only if $S_k=\{i\}$. With the exception being at $S_1$ where it is the first message in the sequence, the probability of selecting message $I(l)$ at the subsequent levels is a joint probability that depends on 1) the user has not yet abandoned at $l-1$ level, $P(W> l-1)=P(W\geq l)$; 2) she has not selected any earlier messages, $\prod_{k=1}^{l-1}(1-u_{I(k)})$; 3) she selects message $I(l)$ when it is displayed, $u_{I(l)}$. 

Given a sequence of messages $\bold{S}=S_1\oplus S_2\oplus\cdots\oplus S_m$, denote $p_a$ as the total abandonment probability over its entire length, which can be expressed as 
$$p_{a}(\bold{S})=\sum_{k=1}^m P(W=k)\prod_{j=1}^{k}(1-u_{I(j)}). $$
It sums over the joint probabilities of not selecting the first $k$ messages and abandoning at the $k^{th}$ level upon receiving the $k^{th}$ unsatisfying message.

\paragraph{Payoff optimization problem} 
Let $U(\bold{S},\bold{u},q)$ denote the total payoff that the platform receives from a given sequence of messages $\bold{S}$ when the valuation is $\bold{u}$ and abandonment follows the geometric distribution with parameter $1-q$. For the simplicity of the notation, we use $U(\bold{S})$ to denote $U(\bold{S},\bold{u},q)$. The expected payoff which the platform is trying to optimize is defined as
$$E[U(\bold{S})]=\sum_{i \in X}p_{i}(\bold{S})r_{i}-cp_{a}(\bold{S}),$$
where $c$ is the cost of losing a customer due to abandonment. 
In contrast to the traditional assortment problems which only focus on revenue maximization, the objective in our model also includes a penalty of losing customer. 

The platform's optimization problem is defined as follows, 
\begin{align}
\max_{\bold{S}}\quad&  E[U(\bold{S})]\label{eq:optimization}\\
s.t. \quad & S_i\cap S_j=\emptyset, \forall i\neq j.\nonumber
\end{align}
The constraint specifies that the sequence cannot contain duplicated messages. It is included to avoid unrealistic solutions where the optimal sequence consisting of identical messages due to the memoryless property of geometric distribution. %Note that if removing the constraint $S_i\cap S_j=\emptyset$, the optimal sequence of messages will be repeated because of the memoryless property of geometric distribution. 
%\todo{As we mentioned earlier, if the user does not select a certain message, it shows she is not interested in it.}
We denote the optimal sequence of messages as $\bold{S}^* = \argmax_{\bold{S}} E[U(\bold{S})]$.

%In the next section, we will first focus on solving the optimization problem when the intrinsic probability $\bold{u}$ and abandonment distribution $F(\cdot)$ are known. In the subsequent sections, we will then investigate how to learn these values based on customer feedback. %In the following two sections, we will first introduce the algorithm to efficiently solve the optimal sequence of assortments and then discuss the learning algorithm.

\iffalse
\textcolor{blue}{
Q: Do we need example to emphasize
\begin{itemize}
\item What's the benefit of retargeting? 
\item What's the benefit of considering fatigue?
\end{itemize}
}
\fi

\section{Characterization of the optimal sequence}
In this section, we describe the algorithm to solve the optimal payoff optimization problem when the valuation $\bold{u}$ and abandonment distribution $F_W$ are both known to the platform. It is an integer programming problem as the platform needs to choose a subset from all available messages and also specify the order. %It is essentially an assignment problem which is NP-hard under the general setting. 
In addition, the choice probability of a particular message $p_{i}(\bold{S})$ depends on its own valuation, as well as the valuation of previous messages shown to the user. This dependence makes the problem much more complicated. We will show in the following result that under the assumption of geometric abandonment distribution, there exists an efficient algorithm for our problem. 

\begin{theorem}\label{T.oneproduct}
 For message $i\in\{1,\cdots,N\}$, define its score as follows,   $$\theta_i:=\frac{r_iu_i-cp(1-u_i)}{1-q(1-u_i)}.$$ Without loss of generality, assume messages are sorted in a decreasing order of their scores, i.e.,  $\theta_1\geq \theta_2\geq \cdots\geq \theta_N$. Then the optimal sequence of messages is $\bold{S}^*=\{1\}\oplus \{2\} \oplus\cdots\oplus\{m\}$, where $m=\max\{i:r_iu_i-cp(1-u_i)> 0\}$.
 \end{theorem}
 Due to the page limit, we only include proof sketches for the key results in the paper. All detailed proofs can be found in the supplementary material.

\emph{Proof sketch:}  We prove Theorem~\ref{T.oneproduct} by contradiction. If the optimal sequence $\bold{S}^*$ is not ordered by the decreasing order of $\theta$, then there exists $S_k^*=\{i\}$, $S_{k+1}^*=\{j\}$ such that $\theta_i<\theta_j$. We compare the payoff generated under this sequence with an alternative sequence whose order of $i$ and $j$ is switched. We show that the alternative sequence generates a higher payoff,  which is a contradiction to the fact that $\bold{S}^*$ is the optimal. $\blacksquare$

%According to the definition, the expected payoff of $\bold{S}^*$ is $E[U(\bold{S}^*)]=\sum_{i=1}^{|\bold{S}^*|}q^{i-1}\prod_{j=1}^i (1-u_{I(j)})\left(r_{I(i)}u_{I(i)}-cp(1-u_{I(i)})\right)$. We prove Theorem~\ref{T.oneproduct} by contradiction. If it is not ordered by the decreasing order of $\theta$, then there exists $S_k=\{i\}$, $S_{k+1}=\{j\}$ such that $\theta_i<\theta_j$. We consider the strategy to switch the order $i$ and $j$. The expected payoffs from other levels will not change except for $k$ and $k+1$. Then the condition $\theta_i<\theta_j$ can be proved to be equivalent to the condition that the expected payoffs from level $k$ and $k+1$ of the new strategy is higher, which is a contradiction to the fact that $\bold{S}^*$ is the optimal. $\blacksquare$

The score $\theta_i$ can be interpreted as follows: $r_iu_i$ is the expected revenue when displaying message $i$, while $cp(1-u_i)$ is its expected abandonment cost. Thus, the numerator denotes the expected payoff of message $i$. 
The denominator is the probability of two events: 1) choose message $i$; 2) abandon the platform after viewing message $i$. Therefore, the score $\theta_i$ is a \emph{normalized} expected payoff, conditioned on the probability {conditional on the event} that message $i$ is making an impact to the payoff. 

Theorem~\ref{T.oneproduct} states that all messages with a positive expected payoff should be  included in the optimal sequence whose order is determined by their scores. Theorem~\ref{T.oneproduct} provides an efficient algorithm with complexity $O(N\log N)$ (where the complexity comes from sorting $N$ messages) and shows that this problem is polynomial-time solvable. %As mentioned earlier, despite being an assignment problem, under the assumption of geometric abandonment distribution, we have proved that this problem is polynomial solvable.

%{\color{red}---- I am a gorgeous divider----}

A special case to Theorem~\ref{T.oneproduct} is when $p=0$, i.e., users never abandon the platform. The following result states that under this scenario,  the optimal sequence only depends on the revenue of the messages.
 
 \begin{proposition}
 With the abandonment probability $p=0$, the optimal sequence is ordered by its revenue. That is, $r_{I(1)}\geq r_{I(2)}\geq\cdots\geq r_{I(N)}$, {where $I$ is the index function of the optimal sequence $\bold{S}^*$.}
 \end{proposition}
With $p=0$, a user will either select one of the messages and generate revenue $r_i$, or leave without any selection after the entire sequence has been shown. Without the risk of user abandonment, the platform can show \emph{all} available messages to a user. In addition, as messages are viewed sequentially, those with higher revenue should have higher priorities and be shown first.

In the next result, we show that we can compare the expected payoff generated under different abandonment distributions if they follow a stochastic order which is stated below for completeness. We want to emphasize that Proposition \ref{P.StochasticOrder} holds under any general distribution for user abandonment, and is not restricted to the geometric distribution. 
 \begin{definition}[Stochastic order]
 Real random variable $W_1$ is stochastically larger than or equal to $W_2$, denoted as $W_1\gtrsim_{s.t.} W_2$, if
 $$P(W_1> x)\geq P(W_2>x) \text{ for all } x\in \mathbb{R}.$$
\end{definition}
 
 \begin{proposition}\label{P.StochasticOrder}
 Assume $\bold{S}'$ and $\bold{S}''$ are the optimal sequences generated under abandonment distribution $W_1$ and $W_2$ respectively. If $W_1\gtrsim_{s.t.} W_2$, we have
 $$E[U(\bold{S}',\bold{u},F_{W_1})]\geq E[U(\bold{S}'',\bold{u},F_{W_2})],$$
 where $U(\bold{S},\bold{u}, F_W)$ denotes the payoff under  strategy $\bold{S}$ when the valuation  and abandonment distribution are $\bold{u}$ and $F_W$ respectively. 
 \end{proposition}
The definition of $W_1\gtrsim_{s.t.} W_2$ implies that users under $F_{W_1}$ are more patient, as they are less likely to abandon the platform upon receiving the same number of unsatisfying messages than their counterparts under $F_{W_2}$. Thus, intuitively, Proposition \ref{P.StochasticOrder}  states that the expected payoff is higher when users are more patient.  

\iffalse
 
{\color{red}---- I am a gorgeous divider (Shall we delete the remaining portion before the next section???----} 
 We will give Example~\ref{E.abandonment} below to illustrate that the abandonment distribution not only influences the contents of the message but also the order to offer.

 \begin{example}[The impact of abandonment distribution]\label{E.abandonment}
There are 10 promotions with profit $r_i=i$ and $u_i=0.05*(11-i)$ for $i\in[1,10]$. Assume $c=2$ and customer's patience follows the geometric distribution with parameter 0.9.  Then the optimal sequential assortment is $\{5\}\oplus\{4\}\oplus\{6\}\oplus\{3\}$.

 For another abandonment distribution that follows the geometric distribution with parameter 0.6, the optimal policy is to offer $\{5\}\oplus\{6\}\oplus\{4\}\oplus\{7\}\oplus\{3\}\oplus\{2\}\oplus\{8\}$. Note that the preference order is changed.
\end{example}

 %\begin{theorem}
%$$E[U(\bold{S}^*_{j})]\geq E[U(\bold{S}^*_{j+1})], \text{ for all } j\in[1,m-1],$$
%\end{theorem}
%\begin{proof}
%\end{proof}

{\color{red}---- I am a gorgeous divider----}
  \fi

 \section{Online learning}
In the previous section, we have assumed that both valuations and user abandonment distribution are known to the platform. It is natural to ask what the platform should do in the absence of such knowledge. In this section we will present an exploration-exploitation algorithm for the SC-Bandit problem and characterize its regret bound. We would like to contrast our method from the traditional bandit settings (e.g., \citealp{auer2002using}): 1) Due to the sequential user behavior and the presence of abandonment, only partial feedback is obtained for learning; 2) The algorithm has to tease out two unknown quantities which jointly influence the user feedback. The aforementioned features of the SC-Bandit makes the analysis of its regret bound much more challenging and involved. 

%{\color{blue}
% At the very beginning, both the valuations of messages and user's abandonment distribution are unknown to the platform. The platform needs to learn customer's preference and patience from mutual interactions. In this section, we introduce an algorithm called Sequential Choice (SC) bandit which explores and exploits simultaneously, to learn valuations and abandonment distribution. Also, we prove that the regret bound of SC bandit is with order $O(\sqrt{NT\log(NT)})$. }

\subsection{Algorithm}
We will present a UCB-motivated algorithm for the SC-Bandit problem to learn the users' valuation $u_i$ for message $i$ and  as well as the abandonment distribution parameter $q$. To characterize the upper confidence bounds, we first identify the unbiased estimators $\hat{u}_i(t)$ and $\hat{q}(t)$ respectively. 

Denote $T_i(t)$ as the total number of  users who observe message $i$ by time $t$ and $c_i(t)$ as the total number of users selecting message $i$. Note that a user does not necessarily observe message $i$ even if $i$ is included in the offered sequence $\bold{S}$ if she abandons the platform before this message is shown. %This uncertainty makes the analysis of regret bound more challenging. 

Let $n_a(t)$ denote the number of users who abandon the platform by time $t$. We use $n_e(t)$ to denote the number of times that users refuse a message {without abandonment} by time $t$. For example, suppose a user at $t=1$ refuses the first two messages and abandons upon receiving the third message, then  $n_e(1) = 2$ and $n_a(1) = 1$. Let $N_q(t)=n_e(t)+n_a(t)$, which denotes the total number of times users turn down unsatisfying messages by time $t$.

 %{\color{blue} SC bandit is a UCB-motivated algorithm but different from the traditional UCB algorithm in the sense that 1) messages included in the offered sequence will not be necessarily explored; 2) there are two uncertain quantities: valuations and abandonment distribution which need to be explored simultaneously; 3) the relation between the regret and the error of estimator is difficult to analyze. 
 %We will give the algorithm and proof of the regret bound below. 
 
%\subsubsection{Unbiased estimators}
%Let $T_i(t)$ denote total times that the user observes the message $i$ before time $t$ and $c_i(t)$ denote total times that the user chooses the message $i$. Note that customer does not necessarily observe message $i$ even if $i$ is included in the offered sequence $\bold{S}$ since they may abandon the system from earlier levels and do not have a chance to observe later messages. This uncertainty makes the analysis of regret bound more challenging. $n_e(t)$ and $n_a(t)$ denote the total times that the customer exits without abandoning and with abandoning from any levels, respectively. Let $N_q(t)=n_e(t)+n_a(t)$, which denotes the total number of exiting behavior.}
%Lemma~\ref{L.unbiased} gives the unbiased estimator for $\bold{u}$ and $q$.The proof we omitted in this section will be provided in the appendix.
 \begin{lemma}[Unbiased estimator]\label{L.unbiased}
 $\hat{u}_i(t)=c_i(t)/T_i(t)$ is an unbiased estimator for $u_i$. Moreover, $\hat{q}(t)=n_e(t)/N_q(t)$ is an unbiased estimator for $q$.
 \end{lemma}
%The proof for Lemma~\ref{L.unbiased} is omitted here and provided in the Appendix.

With Lemma~\ref{L.unbiased}  which gives the unbiased estimators, define the  upper confidence bound for valuation $\bold{u}$ and abandonment distribution parameter $q$ as follows, 
 \begin{equation}\label{E.uvalue}
 u_{i,t}^{UCB}=\hat{u}_i(t)+\sqrt{2\log t/T_i(t)}
 \end{equation}
 and
 \begin{equation}\label{E.qvalue}
 q^{UCB}_t=\hat{q}(t)+\sqrt{2\log t/N_q(t)}.
 \end{equation}
  
Algorithm~\ref{A.promotionfatigue} proposed below is an exploration-exploitation algorithm for the SC-Bandit problem which simultaneously learns valuations and abandonment distribution. For a user arriving at time $t$, we use $u_{i,t-1}^{UCB}$ and $q_{t-1}^{UCB}$ to calculate the current optimal sequence of messages and offer them sequentially to the user. Denote $k_t$ as the last message seen by user $t$, which occurs when one of the following feedback is observed: 1) the user chooses a message; 2) the user abandons the platform; 3) the sequence runs out. We update the upper confidence bound to $u_{i,t}^{UCB}$ and  $q_{t}^{UCB}$ respectively when the last message $k_t$ is shown. 

 \begin{algorithm}
 \textbf{Initialization:} Available messages $X$ with known revenues $\bold{r}$; set $u_{i,0}^{UCB}=1$ for all $i\in X$ and $q_{0}^{UCB}=1$; $n_e(0)=0$; $n_a(0)=0$ ; $t=1$\;
 \While{$t<T$}{
  Compute 
$\bold{S}^t=\argmax_{\bold{S}} \quad E[U(\bold{S},\bold{u}^{UCB}_{t-1},q_{t-1}^{UCB})]$ according to Theorem~\ref{T.oneproduct}\;
  Offer sequence $\bold{S}^t$, observe the user's feedback upon receiving the $k_t$ messages\;
 {
   \For{$i=1:k_t$}{
   update $u_{I(i),t}^{UCB}$ according to Equation~\eqref{E.uvalue}\;
   }
   update $n_e(t),n_a(t)$\;
   update $q_t^{UCB}$ according to Equation~\eqref{E.qvalue}; $t = t+1$\;
    }
 }
 \caption{An exploration-exploitation algorithm for SC-Bandit under marketing fatigue}\label{A.promotionfatigue}
\end{algorithm}

\subsection{Regret bound}
The regret for a policy  $\pi$  is defined as follows, 
%According to the traditional definition of regret bound, define the regret bound of our policy $\pi$ as 
$$Reg_\pi(T;\bold{u},q)=E_\pi\left[\sum_{t=1}^T U(\bold{S}^*,\bold{u},q)-U(\bold{S}^t,\bold{u},q)\right],$$
where $\bold{S}^*$ is the optimal sequence when $\bold{u}$ and $q$ are known to the platform, while $\bold{S}^t$ is the sequence offered to {the user} arriving at time $t$. $E_\pi$ denotes the expectation under the policy $\pi$. %The randomness of the utility comes from the randomness of the user's chosen and abandonment behavior.

To analyze the regret, we first establish the following results. In Lemma~\ref{L.largedeviation}, we provide the concentration analysis of $u_{i,t}^{UCB}$ and $q_t^{UCB}$ using Hoeffding's inequality. %The detailed proof, which is omitted here, will be provided in Appendix. 
%The main idea is to use Hoeffding's inequality to bound $P\left(u_{i,t}^{UCB}<u_i\right)$ and $ P\left(u_{i,t}^{UCB}>u_i+2\sqrt{2\log(Nt)/T_i(t)}\right)$ to analyze deviation of $u_{i,t}^{UCB}$. Similarly, we bound
%$P\left(q^{UCB}_t<q\right)$ and $P\left(q_{t}^{UCB}>q+2\sqrt{2\log(t)/N_q(t)}\right)$
%to analyze deviation of $q_{t}^{UCB}$.
 \begin{lemma}[Concentration bound]\label{L.largedeviation}
 For any $T_i(t)$ and $N_q(t)$, we have
 \begin{align*}
 &P\left(u_{i,t}^{UCB}-\sqrt{8\frac{\log t}{T_i(t)}}<u_i<u_{i,t}^{UCB}\right)\geq 1- \frac{2}{t^4}
 \end{align*}
 and
 \begin{align*}
 &P\left(q^{UCB}_t-\sqrt{8\frac{\log t}{N_q(t)}}<q<q_t^{UCB}\right)\geq 1-\frac{2}{t^4}.
 \end{align*}
 \end{lemma}

Next, Lemma~\ref{L.compare} shows that with the optimal sequence $\bold{S}^*$ determined under  $\bold{u}$ and $q$, its expected payoff is smaller than or equal to the payoff under the same strategy $\bold{S}^*$ when valuation $\bold{u}$ and the abandonment distribution parameter $q$ are higher. Note that this result only holds for $\bold{S}^*$, and does not generally hold for other sequence $\bold{S}$.

%Lemma~\ref{L.compare} states that for the optimal sequence $\bold{S}^*$ computed on  $\bold{u}$ and $q$, the expected utility of $\bold{S}^*$ with $\bold{u}$ and geom($1-q$) is smaller than or equal to the expected utility of $\bold{S}^*$ with higher chosen probability and higher $q$ value. Note that the conclusion does not generally hold for any sequence $\bold{S}$ but only for the optimal sequence $\bold{S}^*$ computed on $\bold{u}$ and $q$.

\begin{lemma}\label{L.compare}
Assume $\bold{S}^*$ is the optimal sequence of messages.  On the condition that $0\leq \bold{u} \leq \bold{u}^{UCB}$ and $0\leq q\leq q^{UCB}$, we have
$$ E[U(\bold{S}^*,\bold{u}^{UCB},q^{UCB})]\geq E[U(\bold{S}^*,\bold{u},q)].$$
\end{lemma}

\emph{Proof sketch:}
Define $E[U(\bold{S}_j)]$ as the expected payoff conditioned on a user entering level $j$. We show that this term can be expressed as  $E[U(\bold{S}_j)]= r_{I(j)}u_{I(j)} + (1-u_{I(j)})\left(qE[U(\bold{S}_{j+1})]-pc\right)$, which is a sum of the expected payoff generated if message $I(j)$ is selected and the future payoff  if $I(j)$ is not selected. Note that the inequality  $r_{I(j)}\geq E[U(\bold{S}_j^*)] \geq q E[U(\bold{S}_{j+1}^*)]-pc$ must hold. Otherwise, removing message $I(j)$ will improve the expected payoff. Using this condition,  we prove by induction  that $ E[U(\bold{S}_j^*,\bold{u}^{UCB},q^{UCB})]\geq E[U(\bold{S}_j^*,\bold{u},q)]$ for all $j$. By definition,  $E[U(\bold{S}_1^*)]  =E[U(\bold{S}^*)]$, and this completes the proof.   $\blacksquare$

%Define $E[U(\bold{S}_j)]$ as the expected utility conditional on the customer entering level $j$, then $E[U(\bold{S}_j)]= r_{I(j)}u_{I(j)} + (1-u_{I(j)})\left(qE[U(\bold{S}_{j+1})]-pc\right)$. The above equality can also be explained as that the payoff from $\bold{S}_j$ consists of three parts: 1) revenue of message $I(j)$ with probability $u_{I(j)}$; 2) payoff from $E[U(\bold{S}_{j+1})]$ with probability $(1-u_{I(j)})q$; 3) penalty cost $c$ on level $j$ with probability $(1-u_{I(j)})p$. For the optimal sequence we have $r_{I(j)}\geq E[U(\bold{S}_j^*)] \geq q E[U(\bold{S}_{j+1}^*)]-pc$, otherwise removing message $I(j)$ will improve the expected payoff. Then we prove by induction from level $|\bold{S}^*|$ to the first level that $ E[U(\bold{S}_j^*,\bold{u}^{UCB},q^{UCB})]\geq E[U(\bold{S}_j^*,\bold{u},q)]$ for each $j$.  $\blacksquare$

Putting everything together, we characterize a regret bound of our online learning algorithm in Theorem~\ref{T.RegretBound}.

\begin{theorem}[Performance bounds for Algorithm \ref{A.promotionfatigue}]\label{T.RegretBound}
Given valuation $u_i$ of message $i$, $i\in X$ and parameter $q$, the regret of the policy during time $T$ is bounded by 
$$Regret_\pi(T;\bold{u},q)= O\left(N\sqrt{T\log T}\right)$$
 where $N$ is the total number of messages.
\end{theorem}

We want to highlight the difficulty of this regret analysis due to the incomplete feedback we observe after a sequence is offered. We are unable to estimate the parameter $q$ if users keep on selecting messages. Meanwhile, we are also unable to estimate the valuation $\bold{u}$ for those messages which are offered but are not seen by a user. The complete proof can be found in the supplementary material.

\emph{Proof sketch:}
Define the ``large'' probability event as $D_t:=\bigcap_{i=1}^N \left(u_{i,t}^{UCB}-\sqrt{8\log t/T_i(t)}<u_i<u_{i,t}^{UCB}\right)
\cap \left(q^{UCB}_t-\sqrt{8\log t/N_q(t)}<q<q_t^{UCB}\right)$.
To bound the regret, we consider the quantity $E[U(\bold{S}^*,\bold{u},q)- U(\tilde{\bold{S}}^t,\bold{u},q)]$ under {on} $D_t$
 and {on }its complement, respectively.

Define $\tilde{\bold{S}}^t$ as the optimal sequence when  valuation is $\bold{u}_t^{UCB}$ and the geometric parameter for the abandonment distribution is $1-q_t^{UCB}$. By the definition of 
$\tilde{\bold{S}}^t$ and $\bold{S}^{*}$, and with Lemma~\ref{L.compare}, we have
$E_\pi[U(\tilde{\bold{S}}^t,\bold{u},q)]\leq E_\pi[U(\bold{S}^{*},\bold{u},q)]\leq E_\pi[U(\bold{S}^*,\bold{u}_t^{UCB},q_t^{UCB})]\leq E_\pi[U(\tilde{\bold{S}}^t,\bold{u}_t^{UCB},q_t^{UCB})]$ on $D_t$. 
Thus, the difference $E_\pi\left[\sum_{t=1}^{T} U(\bold{S}^*,\bold{u},q)-U(\tilde{\bold{S}}^t,\bold{u},q)\right]$
can be bounded above by the expected difference between $U(\tilde{\bold{S}}^t,\bold{u}_t^{UCB},q_t^{UCB})$ and $U(\tilde{\bold{S}}^t,\bold{u},q)$.
This quantity can be further expressed as a sum of two terms which can be analyzed separately, namely, one term is related to the estimated error $q_t^{UCB}-q$, while another is related to the error $(u_{i,t}^{UCB}-u_i)1(i\in \tilde{\bold{S}}^t)$. 

Next, using the coupling method, we bound the error term on $q$. To analyze the regret term of $\bold{u}$, we derive the relation between the probability of exploring message $i$ and the expected regret caused by the error of $u_{i,t}^{UCB}-u_i$. With Lemma~\ref{L.largedeviation}, the regret on $D_t^c$ can be bounded. Combining the regret on $D_t$ and on $D_t^c$, we show that the total regret can be bounded above by $O(N\sqrt{T\log T})$. $\blacksquare$

\section{Personalization with contextual SC-Bandit}
Thus far, we have considered a setting where the platform determines an optimal sequence $\bold{S}^*$ for all its users who share the identical  abandonment distribution and valuations. In this section, we consider a more realistic setting where the abandonment distribution and valuations could differ across users based on some users' context $\bold{x}$. In other words, instead of learning the homogeneous parameter $q$ and $u_i$, the platform needs to learn $q(\bold{x})$ and $u_i(\bold{x})$ which will be used to determine personalized messaging sequences. 

Contextual bandit is an active research area that has received lots of attention in recent years (e.g., \citealp{chu2011contextual,li2010contextual, li2012unbiased,cheung2017thompson}). A common assumption is a linear relationship between the reward and the context. In our setting, since both $q$ and $\bold{u}$ denote probabilities and the observed reward is either 0 or 1, we use the logit model to model $q(\bold{x})$ and $u_i(\bold{x})$ respectively. That is,
\begin{align*}
q(\bold{x})&=e^{\bold{\alpha}^T\bold{x}}/(1+e^{\bold{\alpha}^T\bold{x}}),
\end{align*}
and
\begin{align*}
u_i(\bold{x})&=e^{\bold{\beta}_i^T\bold{x}}/(1+e^{\bold{\beta}_i^T\bold{x}}),
\end{align*}
where $\alpha\in \tilde{\Theta}$ and $\beta_i\in \Theta_i$ are  the unkown parameters to be learnt. $q(\bold{x})$ and $u_i(\bold{x})$ are  generalized linear models (\citealp{mccullagh1989generalized}) as $q(\bold{x})=\mu(\alpha^T\bold{x})$ and $u_i(\bold{y})=\mu(\beta_i^T \bold{x})$, where  $\mu(x)=\exp(x)/(1+\exp(x))$.

Next, we will propose an exploration-exploitation algorithms for the contextual SC-Bandit problem.
%We will propose two exploration-exploitation algorithms for the contextual SC-Bandit problem. The key difference lies in how the ``exploration bonus'' in the algorithm is defined.
%\subsubsection{Contextual SC-Bandit algorithm I}
We adapt the GLM-UCB algorithm proposed by \cite{filippi2010parametric} for our contextual SC-Bandit problem. The key difference from the non-contextual version is that, during each update, we calculate the maximum quasi-likelihood estimator of the parameter $\hat{\beta}_i$, and then update $u_i^{UCB}(\bold{x})$ with $\mu(\hat{\beta}_i^T\bold{x})$ plus an ``exploration bonus" term defined by  $\rho(t)\|\bold{x}\|_{M_{i,t}^{-1}}$, where $\rho(t)$ is a slowly varying function which can be set as $\rho(t)=\sqrt{2\log(t)}$,  $\|v\|_M=\sqrt{v'Mv}$ denotes the matrix norm induced by the positive semidefinite matrix $M$ with 
$M_{i,t}=\lambda \bold{I} + \sum_{k=1}^{t-1} \bold{x}_k\bold{x}'_k 1(\text{user $\bold{x}_k$ observed message $i$})$ where $1(\cdot)$ is the indicator function, and $\lambda$ is a constant.  The update is similar for $q_t^{UCB}(\bold{x})$, i.e., $q^{UCB}_t(\bold{x})=\mu(\hat{\alpha}_i^T\bold{x})+\rho(t)\|\bold{x}\|_{\tilde{M}_{t}^{-1}}$ where $\tilde{M}_{t}=\lambda'\bold{I} + \sum_{k=1}^{t-1} \bold{x}_k\bold{x}_k'n_k$, $n_k$ denotes the number of messages that user $k$ observes, and $\lambda'$ is a constant.

To initialize the algorithm, for the first $N$ users, we offer each of them a message $i$ {where $i$ takes from 1 to $N$}. {In each iteration}, we first update $\hat{\alpha}_{t-1}$ and $\hat{\beta}_{i,t-1}$ based on prior user feedback. Next, we update $q_{t-1}^{UCB}(\bold{x})$ and $u_{i,t-1}^{UCB}(\bold{x})$ for the user $t$ with feature $\bold{x}$. The optimal messaging sequence is obtained by solving the optimization problem $\max_{\bold{S}} E[U(\bold{S},\bold{u}_{t-1}^{UCB}(\bold{x}),q_{t-1}^{UCB}(\bold{x}))]$. For completeness, the GLM-UCB algorithm is given below.

\begin{algorithm}
 \textbf{Initialization:} Available messages $X$ with known revenues $\bold{r}$. Offer each message $1,2,\cdots,N$ to user $1,2, \cdots, N$, observe decision\; 
 Update $M_{i,t},\tilde{M}_{t}$; $t=N$\;
 \While{$t<T$}{
 Update $\hat{\alpha}_t$ and $\hat{\beta}_{i,t}$ by quasi-MLE; $t = t+1$\; 
 Observe customer's contextual information $\bold{x}_t$
  Compute 
$\bold{S}^t=\argmax_{\bold{S}} \quad E[U(\bold{S},\bold{u}^{UCB}_{t-1}(\bold{x}_t),q_{t-1}^{UCB}(\bold{x}_t))]$ according to Theorem~\ref{T.oneproduct} where $q_{t-1}^{UCB}(\bold{x}_t)$ and $\bold{u}_{t-1}^{UCB}(\bold{x}_t)$ are computed by 
$$u_{i,t-1}^{UCB}(\bold{x}_t)=\mu(\hat{\beta}_{i,t-1}^T\bold{x}_t)+\rho(t)\|\bold{x}_t\|_{M_{i,t-1}^{-1}},\forall i$$
and
$$q_{t-1}^{UCB}(\bold{x}_t)=\mu(\hat{\alpha}_{t-1}^T\bold{x}_t)+\rho(t)\|\bold{x}_t\|_{\tilde{M}_{t-1}^{-1}}.$$
%Equation~(\ref{E.qvalueContext}) and 
%(\ref{E.uvalueContext})\;
  Offer personalized messaging sequence $\bold{S}^t$, observe the customer's decision\; 
 {
   %\For{i=1:k}{
  % update $u_{I(i),t}^{UCB}$ according to Equation~(\ref{E.uvalue})\;
  % }
  
   Update $M_{i,t},\tilde{M}_{t}$ \;
    }
 }
 \caption{GLM-UCB algorithm I for contextual SC-Bandit under marketing fatigue}\label{A.PersonalizedFatigue}
\end{algorithm}

\section{Numerical experiments}
In this section, we first investigate the robustness of Algorithm~\ref{A.promotionfatigue} which is our proposed UCB-algorithm for the SC-Bandit problem by comparing how the regret changes with respect to different values of $\bold{u}$. Next, we compare our Algorithm 1 and 2 with two benchmarks in the non-contextual and contextual settings respectively. %which simultaneously explores and exploits with two benchmarks which explicitly separate the two phases.

%The second part of this section compares Algorithm~\ref{A.promotionfatigue} with two benchmark algorithms, which separates the exploration and exploitation periods, to illustrate a significant improvement that Algorithm~\ref{A.promotionfatigue} can make.

\subsection{Robustness of SC-Bandit algorithm}
%The difficulty of learning process may change when: 1) the number of messages included in the optimal sequence changes; 2) the difference of score $\theta_i$ between different messages changes. In this section, we test the robustness of Algorithm~\ref{A.promotionfatigue} through changing the mean and variance of $u_i$.

\paragraph{Experiment setup} We consider a setting with $N = 30$, revenue $r_i$ is uniformly distributed between [0,1], abandonment distribution probability $p = 0.1$ and the cost of abandonment $c = 0.5$. We present four scenarios, when the valuation $\bold{u}$ is  uniformly generated from [0,0.1], [0,0.2], [0.0.3], and [0.0.5], respectively. \\
%We generate 30 random variables uniformly distributed on [0,1] as the revenue of 30 messages. Set the abandonment distribution as geometric with parameter 0.1 and the cost of losing a customer is 0.5. The chosen probability $\bold{u}$ is uniformly generated from [0,0.1], [0,0.2], [0.0.3], respectively.

\noindent{\bf{Result}}  ~~
Figure~\ref{fig:regret2} shows the results  based on 15 independent simulations for different scenarios of $\bold{u}$. The average regrets are 141.13, 121.91, 59.69, and 44.64, respectively. Figure~\ref{fig:regret2} suggests that when $u_i$s are more spread out, it is easier for the algorithm to learn them to a large degree. Meanwhile,  Figure~\ref{fig:regret2}  also reveals something more subtle. When $\bold{u}$ is generated uniformly from [0,0.1],  Algorithm~\ref{A.promotionfatigue} is able to find the optimal sequence before $T=25000$ for a large fraction of the simulations. On the other hand, when $\bold{u}$ is generated uniformly from [0,0.3] or [0,0.5],  the regret continues to increase after the initial 100,000 iterations, indicating that the algorithm has not found the optimal sequence yet. The intuition is that with higher valuations, the length of the optimal sequence could become longer. As the result, it is slower to learn the values of $\bold{u}$ precisely (especially for those messages which are placed later in the sequence), despite learning their approximate values quickly.

%Fig~\ref{fig:regret2} shows that when $\bold{u}$ is generated uniformly from [0,0.1], Algorithm~\ref{A.promotionfatigue} mostly finds the optimal sequence before round 25000, but the regret is higher than the other two cases. The higher regret may due to the smaller difference of chosen probability between messages. When $\bold{u}$ is uniformly generated from $[0,0.3]$ in which case the optimal sequence contains more messages, the regret is much lower in the first 100,000 rounds. However, we notice that the regret curve is still increasing, which indicates that the algorithm still does not find the optimal sequence. Since there are more messages included in the optimal sequence, the learning speed of order will be slower.

\begin{figure}
\centering
  \includegraphics[width=0.8\linewidth]{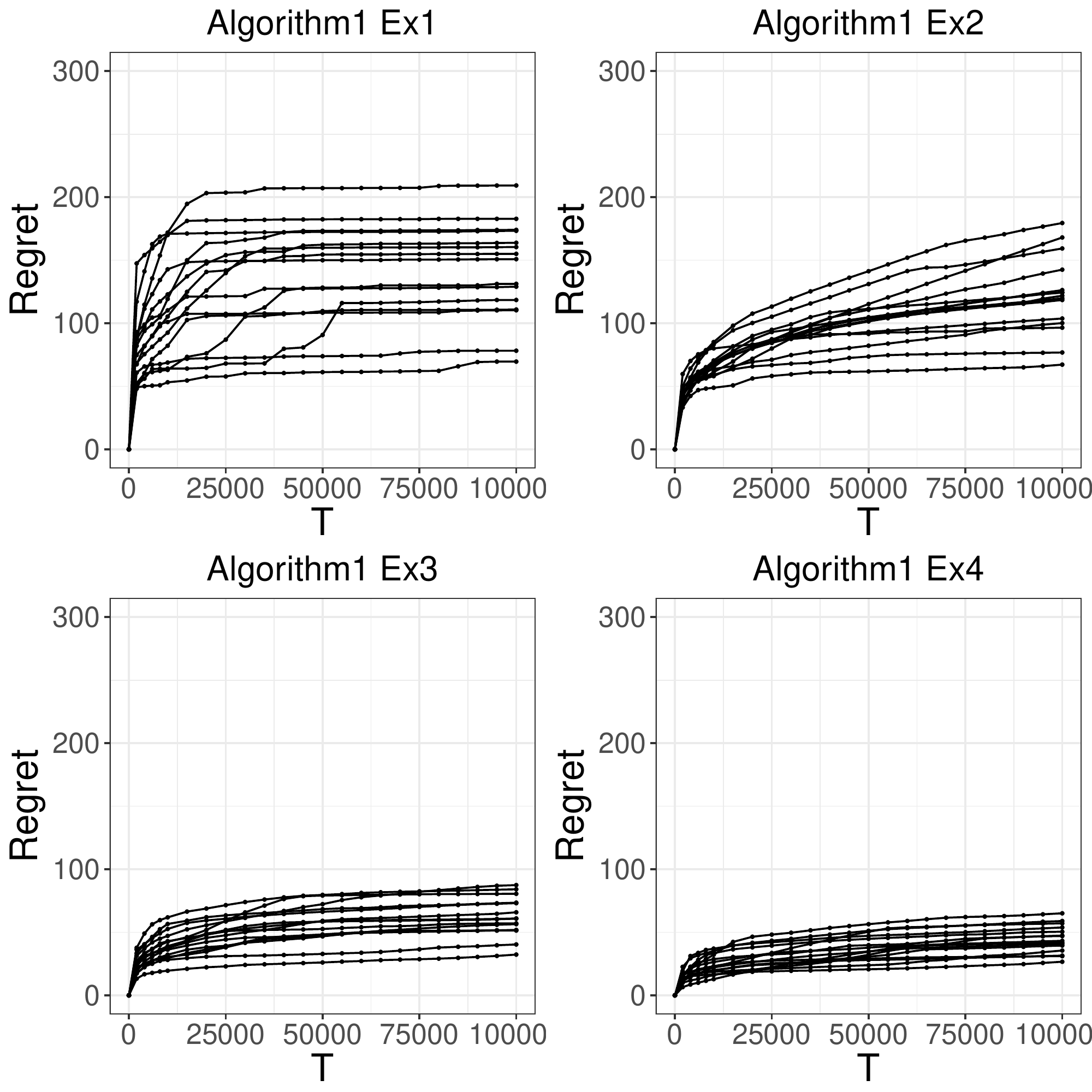}
  \caption{Comparison of Algorithm~\ref{A.promotionfatigue} when $\bold{u}$ is  uniformly generated from [0,0.1], [0,0.2], [0.0.3], and [0.0.5], respectively.}
  \label{fig:regret2}
\end{figure}

\subsection{Comparison with benchmark algorithms}
We analyze two benchmarks and compare their results with our algorithm. The first benchmark is an explore-then-exploit algorithm, while the second ``enhances'' the first benchmark by exploiting the knowledge it has already learned during its exploration phase. \\
%Here we compare our Algorithm~\ref{A.promotionfatigue}, which simultaneously explores and exploits, to two Explore-exploit algorithms. Explore-exploit algorithm refers to an algorithm that separates the exploration and exploitation process. The second Explore-exploit algorithm is an improved algorithm based on the first one.

\noindent{\bf{Benchmark-1}}  With the explore-then-exploit approach, there is an  exploration phase where every message is learnt for at least $\gamma\log(t)$ times during the time period $[0,t]$, where $\gamma$ is a tuning parameter. After this phase, the algorithm uses the estimated parameters to determine an optimal sequence which is offered to all subsequent users. 
We want to highlight that our setting differs from the traditional multi-armed bandit problem where an arm will be pulled if it is selected. In our setting, messages which are to appear later in the sequence may not be viewed by a user. Thus, in order to guarantee that message $i$ is explored, we only offer a single message in a sequence during the exploration phase, i.e., $\bold{S}^t=\{i\}$.\\
%If the message is not explored for $\gamma\log(t)$ times during time period $[0,t]$ where $\gamma$ is a parameter, it will be included in the sequence $\bold{S}^t$ offered at time $t$. However, our message recommendation problem is different than the traditional multi-armed bandit problem where the arm will be pulled if it is selected. Messages arranged at lower levels may not be viewed by the user, so to guarantee the exploration process of message $i$, we set $\bold{S}^t=\{i\}$ if message $i$ needs to be explored.

\noindent{\bf{Benchmark-2}}~~ This algorithm is a variant of Benchmark-{1}. 
During its exploration phase, suppose this benchmark aims to learn the valuation of message $i$. It first solves the optimal sequence problem based on the valuations of the messages which it has already learned, and then appends message $i$ to the beginning of the sequence. Thus, Benchmark-2 learns faster than Benchmark-1 as it offers more messages each time. In addition, it can optimize the sub-sequence to earn higher revenue than its counterpart, making it a competitive baseline.
%During its exploration phase, besides placing the message that needs to be explored on the first level as in Benchmark-{1}, we determine the optimal sequence based on the estimated parameters of the remaining messages and append this sequence to the first message.  
The optimization problem one needs to solve here is nearly identical to (\ref{eq:optimization}) with an additional constraint that $S_1=\{i\}$. It can be proven that the optimal solution has $S_1=\{i\}$ as the first message and the messages of the remaining sequence are ordered according to $\theta_i$ as defined in Theorem~\ref{T.oneproduct}.\\
\iffalse
{\color{blue}----
Separation-based algorithm 1, we propose another algorithm that at the exploration step, besides adding the message that needs to be explored on the first level, we also add other messages with positive expected payoffs to the subsequent levels. It is essentially an optimization problem that maximizes the expected payoff with the constraint that $S_1=\{i\}$, i.e.,
\begin{align*}
    \max_{\bold{S}} \quad& E[U(\bold{S})]\\
    s.t. \quad & S_1=\{i\}\\
    & S_k\cap S_j=\emptyset, \forall k\neq j
\end{align*}

Proposition~\ref{P.rank} gives us the optimal solution to the above optimization problem.

\begin{proposition}\label{P.rank}
The optimal solution to the above optimization problem is $S_1=\{i\}$ and lower levels are ordered decreasingly by positive scores $\theta_j$.
\end{proposition}
\begin{proof}
The proof follows similarly from Theorem~\ref{T.oneproduct}.
\end{proof}}
\fi
%
\paragraph{Experiment setup for SC-Bandit without contexts} We consider a setting that $N = 30$,  $r_i$ is uniformly distributed between [0,1], $p = 0.1$,  $c = 0.5$ and $\bold{u}$ is  uniformly generated from [0,0.1].
%We generate 30 random variables uniformly distributed on [0,1] as the revenue of 30 messages. Set the abandonment distribution as geometric with parameter 0.1 and the cost of losing a customer is 0.5. The chosen probability is uniformly generated from [0,0.1].
%
%Fig~\ref{fig:regret3} shows the regrets for Algorithm~\ref{A.promotionfatigue} from 15 realizations. They all increase as log shape and the average regret is 141.13. The average regret for the Explore-exploit algorithm 2 is 355.39. Noticed from the experiment, the latter algorithm is very sensitive to the parameter that controls the learning speed.
%Fig~\ref{fig:regret} compares three algorithms. It shows that Algorithm~\ref{A.promotionfatigue} performs much better than Explore-exploit algorithm 1 and 2.
%\subsection{SC contextual bandit algorithm}
\paragraph{Experiment setup for SC-Bandit with contexts} We consider a setting with $N=30$, $r_i$ is uniformly distributed between $[0,1]$. The user feature $\bold{x}$ is uniformly generated from $[0,1]^3$. The coefficient related to the abandonment distribution is $\bold{\alpha}=(0.25,0.5,1,0.8)$ where $\alpha_1$ is the intercept. The coefficient related to the valuation of message $i$, $\bold{\beta}_i$, is uniformly generated from $[-2.5,0]^2\times [0,0.5]^2$ where $\beta_{i,1}$ is the intercept.\\

%\paragraph{Result} Fig~\ref{fig:regret_contextual} shows the regret of the three algorithms. 
%\begin{figure}
%\centering
%  \includegraphics[width=0.9\linewidth]{Pic2.pdf}
%  \caption{Comparison between Algorithm~\ref{A.promotionfatigue} and Benchmark2. Mean regret of Algorithm~\ref{A.promotionfatigue} is 141.13 while that of Benchmark is 355.39.}
%  \label{fig:regret3}
%\end{figure}

\begin{figure}
\centering
  \includegraphics[width=0.8\linewidth]{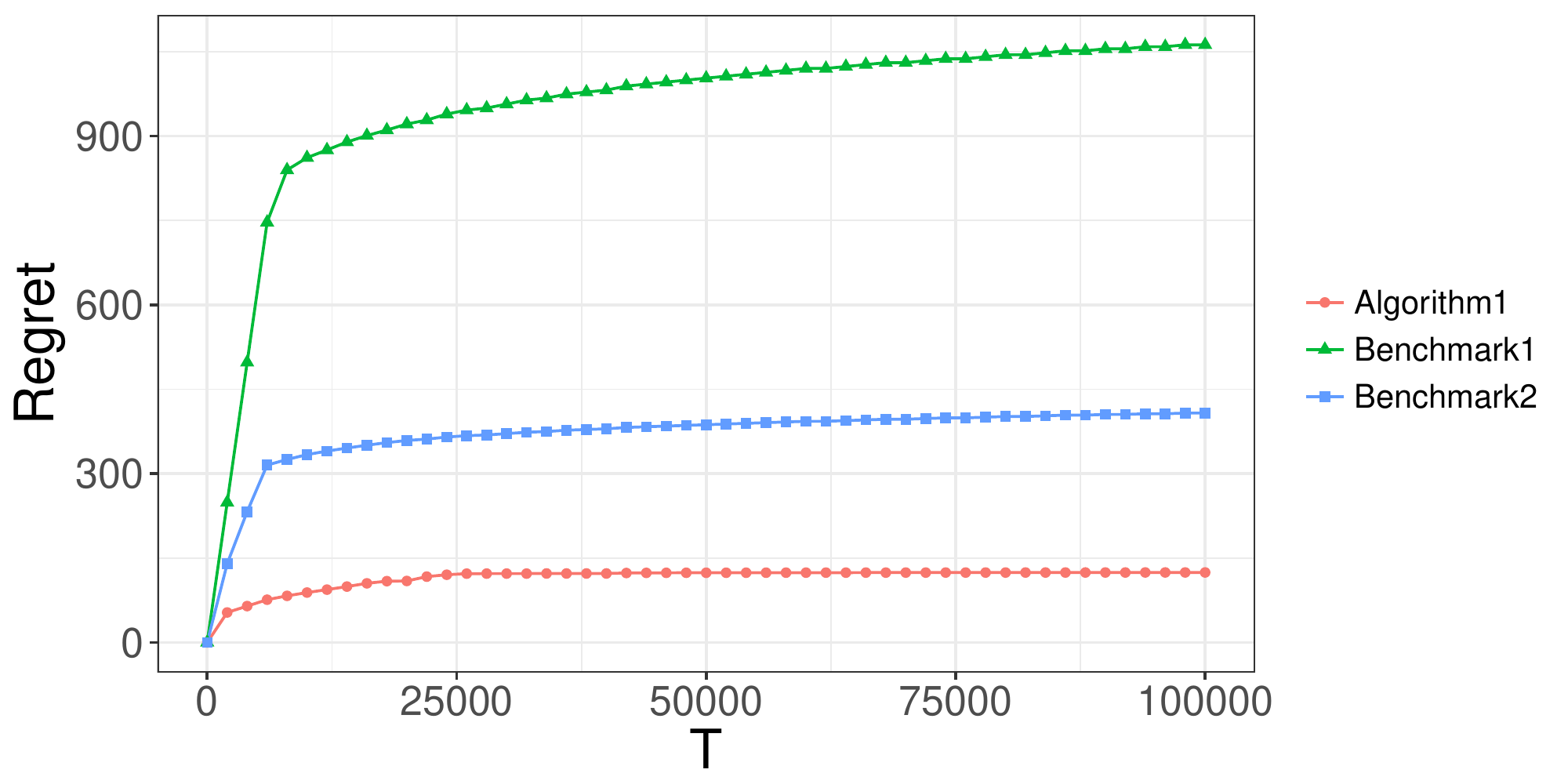}
  \caption{Comparison between Algorithm~\ref{A.promotionfatigue} and two benchmark algorithms in the non-contextual bandit setting.}
  \label{fig:regret}
\end{figure}

\begin{figure}
\centering
  \includegraphics[width=0.8\linewidth]{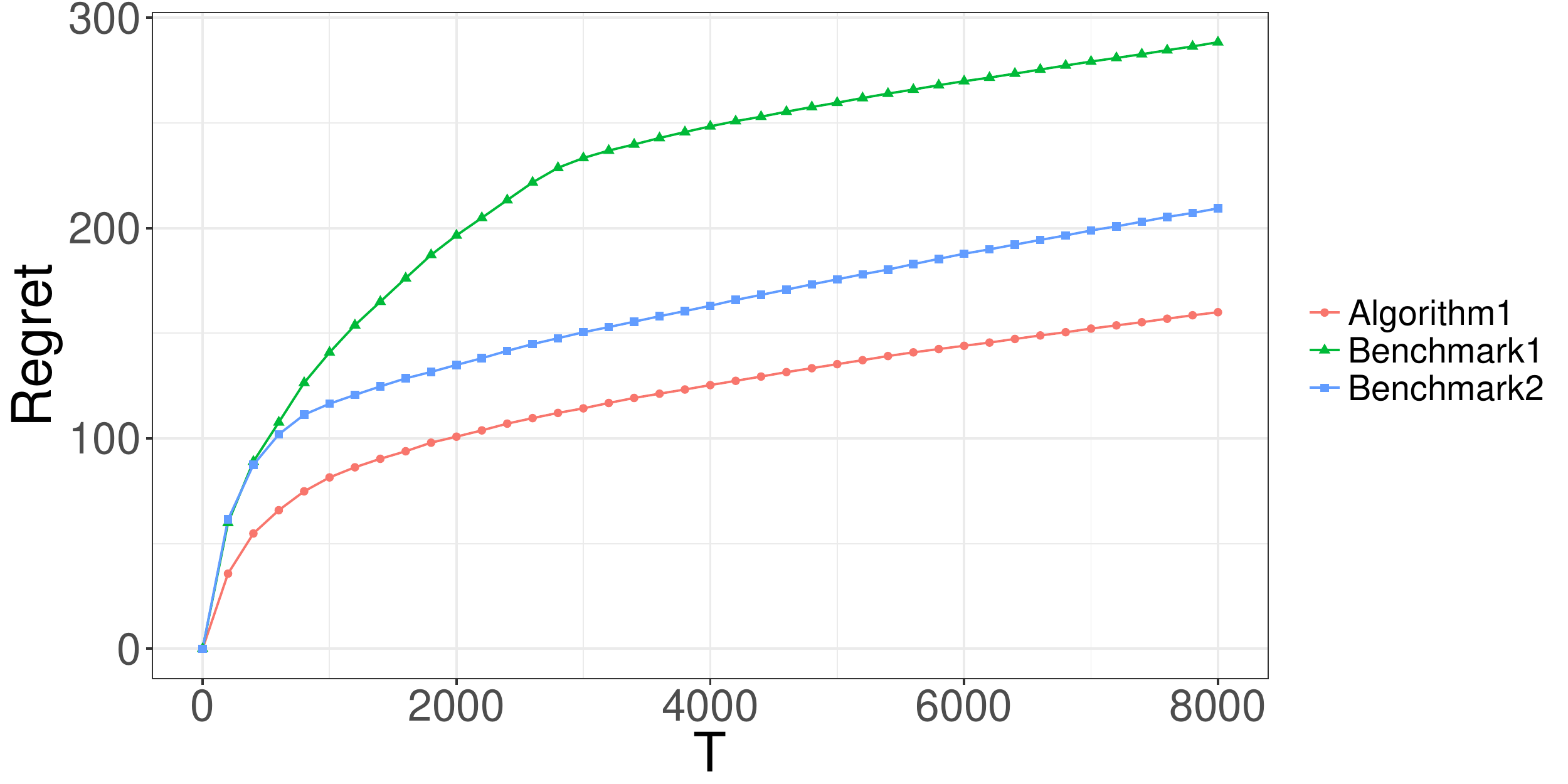}
  \caption{Comparison between Algorithm~\ref{A.PersonalizedFatigue} and two benchmark algorithms in the contextual bandit setting.}
  \label{fig:regret_contextual}
\end{figure}

\noindent{\bf{Result}} ~~
%Figure ~\ref{fig:regret3} shows the results of 15 independent simulations for Algorithm 1 and Benchmark-{2}. 
Figure~\ref{fig:regret} and \ref{fig:regret_contextual} shows the average regret of our algorithm and the two benchmarks under the non-contextual and contextual settings respectively. It is clear that our algorithm outperforms both benchmarks. In particular, Benchmark-{2} does better than Benchmark-{1} as it incorporates learning during its exploration phase.

\section{Conclusion}
In this work, we studied  dynamic learning of a sequential choice bandit problem when users could abandon the platform due to marketing fatigue. We showed that there exists an efficient algorithm to solve the offline optimization problem that determines an optimal sequence of messages to offer to users. For the online learning problem, we proposed an exploration-exploitation algorithm and showed that the resulting regret is bounded by  $O\left(N\sqrt{T\log T}\right)$. Lastly, we proposed a GLM-UCB algorithm to incorporate personalization with user contexts.

There are several future directions of this work. Firstly, as users' preferences may vary over time, it is interesting to incorporate the temporal dimension into the setting. Secondly, different user actions could reveal different levels of interest (e.g., the amount of time a user spent on a message, a user clicked on a message but did not complete a purchase etc.). One question is how to construct and analyze a more accurate user behavior model by utilizing such data. Thirdly, Thompson Sampling would be another natural algorithm to solve the problem we proposed, especially for the personalized version. However, analyzing this setting and providing theoretical results remain a challenging problem.

%{\color{blue}
%\paragraph{Future direction}
%\begin{itemize}
%    \item Users' preference may differ over time, so considering temporal factor may provide more information on behaviour predictions.
%    \item The history of selection behaviour will be an inference of users' preferences and patience, so it may increase the prediction accuracy if it is incorporated in our framework.
%    \item Thompson Sampling would be another efficient algorithm to solve the problem we proposed, especially for the personalized version. However, it remains challenging to analyze.
   % \item As mentioned in Section~\ref{SS.tastes}, it would be interesting to analyze the process of dynamic clustering of users and make full use of the information of other similar users to predict preferences and patience.
%\end{itemize}}

\bibliographystyle{apalike}  
\bibliography{fatigue}  %%% Remove comment to use the external .bib file (using bibtex).


\end{document}